\documentclass[12pt,letterpaper]{article}
\usepackage{amssymb,amsmath,amsthm,enumerate}
\usepackage{graphicx, color}
    \usepackage[driverfallback=hypertex,pagebackref=true,colorlinks]{hyperref}
    \hypersetup{linkcolor=[rgb]{.7,0,0}}
    \hypersetup{citecolor=[rgb]{0,.7,0}}
    \hypersetup{urlcolor=[rgb]{.7,0,.7}}
\usepackage{geometry}
\usepackage{epstopdf}
\AtBeginDocument{%
  \addtolength\abovedisplayskip{-0.25\baselineskip}%
  \addtolength\belowdisplayskip{-0.15\baselineskip}%
}

\usepackage{float}
\floatstyle{ruled}
\newfloat{algorithm}{tbp}{loa}
\providecommand{\algorithmname}{Algorithm}
\floatname{algorithm}{\protect\algorithmname}

\usepackage[titletoc,title]{appendix}
\usepackage{url}

\newtheorem{theorem}{Theorem}
\newtheorem{lemma}{Lemma}[section]

\newtheorem{definition}[lemma]{Definition}

\newtheorem{remark}[lemma]{Remark}

\def\bits{ \{0,1\} }
\def\uint{ [0,1] }
\def\N{{\mathbb {N}}}
\def\Reals{{\mathbb {R}}}

\newcommand{\Ex}{\mathop{\bf E\/}}

\def\A{{\mathcal {A}}}

\def\U{{\mathcal {U}}}

\def\T{{\mathcal {T}}}
\def\L{{\mathcal {L}}}

\def\dom{{\rm domain}}
\def\im{{\rm image}}

\def\P{{\mathbb {P}}}

\title{Fast Learning Requires Good Memory:\\
A Time-Space Lower Bound for Parity Learning}
\author{Ran Raz%
\thanks{Weizmann Institute of Science, Israel, and the Institute for Advanced Study, Princeton, NJ. Research supported by the Israel Science Foundation grant No. 1402/14, by the I-CORE Program of the Planning and Budgeting Committee and the Israel Science Foundation, by the Simons Collaboration on Algorithms and Geometry, by the Fund for Math at IAS, and by the National Science Foundation grant No. CCF-1412958. Any opinions, findings and conclusions or recommendations expressed in this material are those of the author and do not necessarily reflect the views of the National Science Foundation. Email: \texttt{ran.raz.mail@gmail.com}}}

\date{}
\begin{document}
\maketitle

\begin{abstract}
We prove that any algorithm for learning parities requires either a memory
of quadratic size or an exponential number of samples.
This proves a recent conjecture of Steinhardt, Valiant and Wager~\cite{SVW}
and shows that for some learning problems a large storage space is crucial.

More formally,
in the problem of parity learning, an unknown string $x \in \bits^n$ was chosen uniformly at random. A learner tries to learn $x$ from a stream of samples
$(a_1, b_1), (a_2, b_2) \ldots$, where each~$a_t$ is uniformly distributed over $\bits^n$ and  $b_t$ is the inner product of $a_t$ and $x$, modulo~2.
We show that any algorithm for parity learning, that uses
less than $\frac{n^2}{25}$ bits of memory, requires an exponential number of samples.

Previously, there was no non-trivial lower bound
on the number
of samples needed, for any learning problem, even if the
allowed memory size is $O(n)$
(where $n$ is the space needed to store one sample).

We also give an application of our result in the field of bounded-storage cryptography. We show an encryption scheme that requires a private key of length $n$, as well as time complexity of $n$ per encryption/decription of each bit, and is
provenly and unconditionally secure as long as the attacker uses
less than $\frac{n^2}{25}$ memory bits
and the scheme is used at most an exponential
number of times.
Previous works on bounded-storage
cryptography assumed that the memory size used by the attacker
is at most linear in the time
needed for encryption/decription.
\end{abstract}
\newpage

\section{Introduction}

Parity learning can be solved in polynomial time,
by Gaussian elimination, using $O(n)$ samples
and  $O(n^2)$ memory bits. On the other hand, parity learning can be solved
by trying all the possibilities, using $n+ o(n)$ memory bits and an exponential
number of samples.

We prove that any algorithm for parity learning requires either
$\frac{n^2}{25}$ memory bits, or an exponential number of samples.
Our result may be of interest from the points of view of
learning theory, computational complexity and cryptography.

\subsection{Learning Theory}

The main message of this paper from the point of view of learning theory is that
for some learning problems, access to a relatively large memory is crucial.
In other words, in some cases, learning is infeasible, due to memory constraints.
We show that there exist concept classes that can be efficiently learnt from a
polynomial number of samples, if the learner
has access to a quadratic-size memory, but require an exponential number
of samples
if the memory used by the learner is of less than quadratic size.
This gives a formally stated and mathematically proved example for the intuitive
feeling that a "good" memory may be very helpful in learning processes.

Many works studied the resources needed for learning, under
certain information,
communication or memory constraints
(see in particular~\cite{Shamir, SVW} and the many references given there).
However, there was no previous non-trivial lower bound on the
number of samples needed, for any learning problem,
even when
the allowed memory size is bounded by
the length of one sample
(where we don't count the space taken
by the current sample that is being read).

The starting point of our work is the
intriguing recent work of Steinhardt, Valiant and Wager~\cite{SVW}.
Steinhardt, Valiant and Wager asked whether there exist
concept classes that can be efficiently learnt from a
polynomial number of samples, but
cannot be learnt from a
polynomial number of samples
if the allowed memory size  is linear in the length of one sample.
They conjectured that the problem of parity learning provides such a separation.
Our main result proves that conjecture.

\begin{remark} \label{rem1}
Conjecture~1.1 of~\cite{SVW} conjectures that any algorithm for
parity learning requires either at least $\frac{n^2}{4}$ bits of memory,
or at least $2^{n/4}$ samples.
Our main result qualitatively proves this conjecture, but with different constants.
The conjecture, as stated,
(that is, with the ambitious constants  $\frac{1}{4}, \frac{1}{4} $)
is too strong.\footnote{Roughly speaking, \label{footnote}
this is the case since an algorithm similar to
Gaussian elimination can solve parity learning, using
$\frac{n^2}{4} + O(n)$ memory bits and a polynomial
number of samples (by keeping in step~$k$, a matrix with $k$ rows
and $n$ columns, where the first $k$ columns form the identity matrix).
If $2^{n/4}$ samples are available, one can essentially solve a parity learning
problem of size $\frac{3}{4}n + o(n)$, by considering only samples
with coefficients~0 on the last $\frac{1}{4}n - o(n)$ variables.
Hence, if $2^{n/4}$ samples are available,
$\frac{9}{64}n^2 + o(n^2)$ memory bits are sufficient.}
\end{remark}

\subsection{Computational Complexity}

Time-space tradeoffs
have been extensively studied in the field of computational complexity,
in many works and various settings.
Two brilliant lines of research were particularly successful in establishing time-space
lower bounds for computation.

The first line of works~\cite{BJS,Ajtai1,Ajtai2,BSSV}
gives explicit examples for
polynomial-time computable
Boolean functions
$f:\bits^n \rightarrow \bits$, such that, any algorithm
for computing $f$ requires either at least $n^{1-\epsilon}$ memory bits,
where $\epsilon > 0$ is an arbitrarily  small constant, or
time complexity of at least $\Omega\left(n \sqrt{\log n/\log \log n}\right)$.
These bounds are proved for any {\em branching program} that computes $f$.
Branching programs are the standard
and most general computational model for studying time-space tradeoffs
in the {\em non-uniform} setting (which is the more general setting),
and is also the computational model
that we use in the current work.

The second line of works~\cite{Fortnow,FLvMV,Williams,Williams2}
(and other works)
studies time-space tradeoffs for $SAT$ (and other $NP$ problems), in the
{\em uniform}
setting,
and proves that any algorithm
for $SAT$ requires either at least $n^{1-\epsilon}$ memory bits, or
time complexity of at least $n^{1 +\delta}$
(where $0 < \epsilon, \delta < 1$ are constants).
For an excellent survey, see~\cite{Melkebeek}.

Both lines of works obtain less than quadratic lower bounds
on the time needed for computation, under memory constraints.
Quadratic lower bounds
on the time needed for computation are not known, even if
the allowed memory-size
is logarithmic.
Comparing these results to our work, one may ask
what makes it possible to prove
exponential lower bounds on the time needed for parity learning,
under memory constraints, while the known time-space lower bounds
for computations are significantly weaker?
The main point to keep in mind is that when studying time-space tradeoffs
for computing a function, one assumes that the input for the function
can always be accessed, and the space needed to store the input doesn't count
as memory that is used by the algorithm. Thus, the input is stored for free.
In our learning problem, it is assumed
that after the learner saw a sample, the learner cannot access
that sample again,
unless the sample was stored in the learner's memory.
The learner can always get a new sample that is "as good as the old one", but
she cannot access the same sample that she saw before (without storing it in the memory).

Finally, let us note that by Barrington's celebrated result,
any function in $NC$ can be computed by a polynomial-length
branching program of width~5~\cite{Barrington}.
Hence, proving
super-polynomial lower bounds on the time needed for computing a function,
by a branching program of width~5, would imply
super-polynomial lower bounds for formula size.

\subsection{Cryptography}

Assume that a group of (two or more)
users share a (random) secret key $x \in \bits^n$.
Assume that user Alice wants to send an encrypted bit
$M \in \bits$ to user Bob.
Let~$a$ be a string of $n$ bits, uniformly distributed over $\bits^n$,
and assume that both Alice and Bob know $a$
(we can think of $a$ as taken from a shared random string and if a shared
random string is not available Alice can just choose $a$ randomly
and send it to Bob).
Let  $b$ be the inner product of $a$ and $x$, modulo~2.
Thus, $b$ is known to both Alice and Bob and can be used as a one time pad
to encrypt/decrypt $M$, that is, Alice encrypts by computing $M \oplus b$ and Bob
decrypts by computing
$M=(M \oplus b) \oplus b$.

Assume that this protocol is used
$m+1$ times, with the same secret key $x$, where $m$ is
less than exponential. Denote by
$a_t, b_t$ the string $a$ and bit $b$ used at time~$t$.
Suppose that during all that time,
an attacker could see $(a_1, b_1), \ldots ,(a_m, b_m)$,
but the attacker
has less than $\frac{n^2}{25}$ bits of memory.
Our main result shows that the attacker cannot guess the secret key $x$,
with better than exponentially small probability.
Therefore, using the fact that inner product is a strong extractor (with
exponentially small error), even if the attacker  sees $a_{m+1}$,
the attacker cannot
predict $b_{m+1}$, with better than exponentially small advantage over a random
guess.

Thus, if the attacker
has less than $\frac{n^2}{25}$ bits of memory, the encryption remains secure
as long as it is used less than an exponential number of times.

{\em Bounded-storage cryptography}, first introduced by Maurer~\cite{Maurer}
and extensively studied in many works,
studies cryptographical protocols that are
secure under the assumption that the memory used by the attacker is limited
(see for example~\cite{CM,AR,ADR,Vadhan,DM}, and many other works).
Previous works on bounded-storage cryptography assumed the existence
of a high-rate source of randomness that streams random bits to all parties.
The main idea is that the attacker doesn't have sufficiently large memory
to store all random bits, and hence a shared secret key can be used
to randomly select (or extract) bits from the random source that the attacker
has very little information about.

In previous works, the number of random bits transmitted
during the encryption was assumed to be larger than the memory-size
of the attacker.
Thus, the time needed for encryption/decryption was at least linear in the
memory-size
of the attacker.
In contrast, the time needed for encryption/decryption in our protocol is $n$,
while the encryption is secure against attackers with memory of size
$\frac{n^2}{25}$.

\begin{remark}
If Alice and Bob want to transmit encrypted messages of length $m$,
where $m \geq n$ (and the attacker has $O(n^2)$ bits of memory), our protocol
has no advantage over previous ones, as the time needed for
encryption/decription in our protocol is $mn$.
The advantage of our protocol is in situations where
the users want to securely transmit many shorter messages.
\end{remark}

\subsection{Our Result}

\subsubsection*{Parity Learning}

In the problem of parity learning, there is an unknown string $x \in \bits^n$ that was chosen uniformly at random. A learner tries to learn $x$ from samples
$(a, b)$, where $a \in_R \bits^n$ and $b = a \cdot x$
(where $a \cdot x $ denotes inner product modulo 2). That is, the learning algorithm is given a stream of samples,
$(a_1, b_1), (a_2, b_2) \ldots$, where each~$a_t$ is uniformly distributed over $\bits^n$ and for every $t$, $b_t = a_t \cdot x$.

\subsubsection*{Main Result}

\begin{theorem} \label{thm:TM1}
For any $c < \frac{1}{20}$, there exists $\alpha >0$, such that the following holds:
Let $x$ be uniformly distributed over $\bits^n$.
Let $m \leq 2^{\alpha n}$.
Let $A$ be an algorithm that
is given as input a stream of samples,
$(a_1, b_1),  \ldots, (a_m, b_m)$, where each~$a_t$ is uniformly distributed over $\bits^n$ and for every $t$, $b_t = a_t \cdot x$.
Assume that $A$ uses at most $c n^2$ memory bits and outputs
a string $\tilde{x} \in \bits^n$.
Then, $\Pr[\tilde{x} = x] \leq O(2^{-\alpha n})$.
\end{theorem}

Theorem~\ref{thm:TM1} is restated, in a
stronger\footnote{Theorem~\ref{thm:TM2} allows the algorithm to output
an affine subspace of dimension $\leq \frac{3}{5} n$, and bounds
by $2^{-\alpha n}$
the probability that $x$ belongs to that affine subspace.}
and more formal~\footnote{Theorem~\ref{thm:TM2} models the algorithm
by a branching program, which is more formal and clarifies that the theorem
holds also in the (more general) non-uniform setting.} form,
as Theorem~\ref{thm:TM2} in
Section~\ref{sec:TMproof}, and the proof of Theorem~\ref{thm:TM2} is given there.

\section{Preliminaries}

For an integer $n$, denote $[n] = \{1,\ldots,n\}$.
For $a,x \in \bits^n$, denote by $a \cdot x $ their inner product modulo 2.

For a function $P: \Omega \rightarrow \Reals$, we denote by $\left|P\right|_1$ its $\ell_1$ norm.
In particular, for two distributions, $P,Q: \Omega \rightarrow \uint$, we denote by $\left|P-Q\right|_1$ their $\ell_1$ distance.

For a random variable $X$ and an event $E$,
we denote by $\P_X$ the distribution of the random variables $X$, and
we denote by $\P_{X|E}$ the distribution of the random variable $X$ conditioned on the event $E$.

Denote by $\U_n$ the uniform distribution over $\bits^n$.
For an affine subspace $w \subseteq \bits^n$, denote by $\U_w$ the uniform distribution over $w$.

For $n \in \N$,
denote by $\A(n)$ the set of all affine subspaces of $\bits^n$.

\section{Proof Outline}

\subsubsection*{Computational Model}

We model the learning algorithm by a {\it branching program}.
A branching program of length $m$ and width $d$, for parity learning, is a directed (multi) graph with vertices arranged in $m+1$ layers containing at most $d$ vertices each. Intuitively, each layer represents a time step and each
vertex represents a memory state of the learner.
In the first layer, that we think of as layer~0, there is only one vertex, called the start vertex.
A vertex of outdegree~0 is called a  leaf.
Every non-leaf vertex in the program has $2^{n+1}$ outgoing edges, labeled by elements
$(a,b) \in \bits^n \times \bits$, with exactly one edge labeled by each such $(a,b)$, and all these edges going
into vertices in the next layer.
Intuitively, these edges represent the action when reading $(a_t,b_t)$.
The samples
$(a_1, b_1), \ldots, (a_m, b_m) \in \bits^n \times \bits$
that are given as input,
define a
computation-path in the branching
program, by starting from the start vertex
and following at
Step~$t$ the edge labeled by~$(a_t, b_t)$, until reaching a leaf.

Each leaf $v$ in the program is labeled by an affine subspace $w(v) \in \A(n)$, that
we think of as the output of the program on that leaf.
The program outputs the label $w(v)$ of the leaf $v$ reached by the computation-path.
We interpret the output of the program as a guess that $x \in w(v)$.

We also consider {\em affine branching programs},
where every vertex $v$ (not
necessarily a leaf) is labeled by an affine subspace $w(v) \in \A(n)$, such that,
the start vertex is labeled by the space $\bits^n \in \A(n)$, and
for any edge $(u,v)$, labeled by $(a,b)$, we have
$w(u) \cap \{x' \in \bits^n : a \cdot x' = b \} \subseteq w(v)$.
These properties guarantee that if the computation-path reaches
a vertex $v$ then $x \in w(v)$.
Thus, we can
interpret $w(v)$ as an affine subspace that  is known to contain $x$.

An affine branching program is called {\em accurate} if for
(almost) all vertices $v$, the distribution of $x$, conditioned on the
event that the computation-path reached $v$, is close to the uniform
distribution over $w(v)$.

For exact definitions, see Section~\ref{section:def}.

\subsubsection*{The High-Level Approach}

The proof has two parts. We prove lower bounds for affine branching programs,
and we reduce general branching programs to affine branching programs.
The hard part is the reduction from general branching programs to affine
branching programs. We note that this reduction is very wasteful and expands
the width of the branching program by a factor of $2^{\Theta(n^2)}$.
Nevertheless, since we allow our branching program to be
of width up to $2^{O(n^2)}$, this is still affordable (as long as  the exact
constant in the exponent is relatively small).
We have to make sure though that, 
when proving time-space lower bounds for affine branching programs, 
the upper bounds
that we assume on the width
of the affine
branching programs are larger than the expansion of the width caused by the reduction.

We note that in the introduction to Conjecture~1.1 of~\cite{SVW},
Steinhardt, Valiant and Wager mention that they were able to prove
the conjecture ``for any algorithm whose memory states correspond to subspaces''.
However, a formal statement (or proof) is not given, so we do not know how similar their result is to our lower bound for affine branching programs.
We note that affine branching programs, as we define here, do not satisfy
Conjecture~1.1 of~\cite{SVW} (see Remark~\ref{rem1} and Footnote~\ref{footnote}).

\subsubsection*{Lower Bounds for Affine Branching Programs}

Assume that we have an affine branching program of 
length at most $2^{cn}$ and width at most $2^{cn^2}$,
for a small enough constant $c$.
Fix $k = \frac{4}{5}n$. 
We 
prove that the probability that the 
computation-path reaches some vertex that is labeled with an affine
subspace of dimension $ \leq k$ is at most
$2^{-\Omega(n^2)}$.

Without loss of generality, we can assume that all vertices in the program are 
labeled with affine subspaces of dimension $\geq k$.
Other vertices can just be removed as the computation-path must reach a vertex
labeled with a subspace
of dimension $k$, before it reaches a vertex
labeled with a subspace
of dimension $<k$ (because the dimension can decrease by at most 1 along an edge).

We define the ``orthogonal'' to an affine subspace as the vector space orthogonal
to the vector space that defines that affine subspace (that is, the vector space
that the affine subspace is given as it's translation).

Let $v$ be a vertex in the program, such that, $w(v)$ is 
of dimension~$k$.
It's enough to prove that 
the probability that the computation-path reaches $v$ is at most
$2^{-\Omega(n^2)}$.

To prove this, we consider the 
vector spaces ``orthogonal'' to the 
affine subspaces that label the vertices
along the computation-path, and for each of them we consider its intersection
with the vector space ``orthogonal'' to $w(v)$. We note that, in each step,
the probability that the dimension of the intersection increases is exponentially
small (as it requires that the $a_t$ currently being read is 
contained in some small vector space). 
Since the dimension of the intersection must increase a linear number of times, in
order for the computation-path to reach $v$, a simple union bound shows that
the probability to reach $v$ is at most $2^{-\Omega(n^2)}$.

The full details are given in Lemma~\ref{lemma:probaffine}.

\subsubsection*{From Branching Programs to Affine Branching Programs}

In Section~\ref{sec:simulation}, we show how to simulate a branching program
by an accurate affine branching program.
We do that layer after layer.
Assume that we are already done with layer $j-1$, so every vertex in
layer $j-1$ is already labeled by an affine subspace, and
the distribution of $x$, conditioned on the
event that the computation-path reached a vertex, is close to the uniform
distribution over the affine subspace that labels that vertex.

Now, take a vertex $v$ in layer $j$, and consider
the distribution of $x$, conditioned on the
event that the computation-path reached the vertex $v$.
By the property that we already know on layer $j-1$,
this distribution is close to a
convex combination of uniform distributions over affine subspaces of $\bits^n$.

One could split $v$ into a large number of vertices, one vertex for each
affine subspace in the combination. However, this practically means
that we would have a vertex for any affine subspace. We would like to keep
the number of vertices somewhat smaller.
This is done by grouping many affine subspaces into one group. The group will
be labeled by an affine subspace that contains all the affine subspaces in the group.
Moreover, we will have the property that for each such group, the uniform
distribution over the affine subspace that labels the group is close to the
relevant weighted average of the uniform
distributions over the affine subspaces in the group.
Thus, practically, we can replace all the affine subspaces in the group by
one affine subspace that represents all of them.

Lemma~\ref{lemma:space-representation} shows that it is possible to group
all the affine subspaces into a relatively small number of groups.

We note that the entire inductive argument is delicate, as we cannot afford
deteriorating the error multiplicatively in each step and 
need to make sure that all errors are additive.

\section{Distributions over Affine Subspaces} \label{sec:dist}

In this section, we study convex combinations of uniform distributions over affine subspaces of $\bits^n$.
Lemma~\ref{lemma:space-representation}
is the only result, proved in this section,  that is used outside the section.

In the following lemmas, we have a random variable $W \in \A(n)$ and we consider the distribution $\Ex_{W} [\U_W]$. This distribution is a convex combination of uniform distributions over affine subspaces of $\bits^n$.

The first lemma identifies a condition that implies that the distribution $\Ex_{W} [\U_W]$ is close to the uniform distribution over $\bits^n$.

\begin{lemma} \label{Lemma:Fourier}
Let $W \in \A(n)$ be a random variable. Let $r \geq \frac{n}{2}$. Assume that for every $a \in \bits^n$, such that $a \neq \vec{0}$, and every $b \in \bits$,
$$\Pr_W [ \forall x \in W: a \cdot x  = b  ] \leq 2^{-r}.$$
Then
$$\left| \Ex_{W} [\U_W] - \U_n \right|_1 < 2^{-\left(r-\frac{n}{2}\right)}.$$
\end{lemma}

\begin{proof}
The proof uses Fourier analysis.
For any affine subspace $w \subseteq \bits^n$, the Fourier coefficients of $\U_w$ are:
$$\widehat{\U_w}(a) = \left\{
\begin{array}{cc}
  2^{-n} &  \mbox{  if } \forall x \in w: a \cdot x  = 0  \\
  -2^{-n} &  \mbox{  if } \forall x \in w: a \cdot x  = 1  \\
  0 &  \mbox{  otherwise }
\end{array} \right.
$$
Hence, the Fourier coefficients of $\Ex_{W} [\U_W]$ are:
$$\widehat{\Ex_{W} [\U_W]}(a) = 2^{-n} \cdot
\left( \Pr_W [ \forall x \in W: a \cdot x  = 0  ] -
\Pr_W [ \forall x \in W: a \cdot x  = 1  ]
\right),
$$
and note that this also implies
$$\widehat{\Ex_{W} [\U_W]}(\vec{0}) = 2^{-n}.
$$
The Fourier coefficients of $\U_n$ are:
$$\widehat{\U_n}(a) = \left\{
\begin{array}{cc}
  2^{-n} &  \mbox{  if } a = \vec{0}   \\
  0 &  \mbox{  if }  a \neq \vec{0}
\end{array} \right.
$$
Thus,
$$
\sum_{a \in \bits^n} \left( \widehat{\Ex_{W} [\U_W]}(a) - \widehat{\U_n}(a) \right)^2
<
2^{n} \cdot \left( 2^{-n} \cdot 2^{-r} \right)^2
=
2^{-n -2r}.
$$
By Cauchy-Schwarz and Parseval,
$$
\left( \Ex_{x \in_R \bits^n} \left| {\Ex_{W} [\U_W]}(x) - {\U_n}(x) \right| \right)^2
\leq
\Ex_{x \in_R \bits^n} \left( {\Ex_{W} [\U_W]}(x) - {\U_n}(x) \right)^2
=
$$
$$
\sum_{a \in \bits^n} \left( \widehat{\Ex_{W} [\U_W]}(a) - \widehat{\U_n}(a) \right)^2
<
2^{-n -2r}.
$$
Therefore,
$$
\left| \Ex_{W} [\U_W] - \U_n \right|_1 =
2^n  \Ex_{x \in_R \bits^n} \left| {\Ex_{W} [\U_W]}(x) - {\U_n}(x) \right|
<
2^n \cdot \sqrt{2^{-n -2r}}=
2^{-(r-n/2)}.
$$
\end{proof}

The next lemma shows that always there exists an affine subspace
$s \subseteq \bits^n$, such that the distribution
$\Ex_{W \mid (W \subseteq s) } [\U_W]$ is close to the uniform distribution over $s$, and the event $W \subseteq s$ occurs with non-negligible probability.

\begin{lemma} \label{lemma:subspace}
Let $W \in \A(n)$ be a random variable. Let $r \geq \frac{n}{2}$.
There exists an affine subspace $s \subseteq \bits^n$, such that:
\begin{enumerate}
\item
$$\Pr_W [ W \subseteq s ] \geq
2^{- \sum_{i=0}^{n-\dim (s) - 1} \left(r - \frac{i}{2} \right)} .$$
\item
$$\left| \Ex_{W \mid (W \subseteq s) } [\U_W] - \U_s \right|_1 <
2^{-\left(r-\frac{n}{2}\right)}.$$
\end{enumerate}
\end{lemma}

\begin{proof}
The proof is by induction on $n$.
The base case, $n=0$, is trivial, because in this case the only element of $\A(n)$ is $\{ \vec{0} \}$, so the lemma follows with $s=\{ \vec{0} \}$.

Let $n \geq 1$.
If for every $a \in \bits^n$, such that $a \neq \vec{0}$, and every $b \in \bits$, we have
$\Pr_W [ \forall x \in W: a \cdot x  = b  ]  \leq 2^{-r}$, the proof follows by Lemma~\ref{Lemma:Fourier}, with $s = \bits^n$.
Otherwise, there exists $a \neq \vec{0}$, and $b \in \bits$, such that,
$\Pr_W [ \forall x \in W: a \cdot x  = b  ] > 2^{-r}$.
Denote by $u$ the $(n-1)$-dimensional affine subspace
$$u = \{ x \in \bits^n : a \cdot x  = b\}.$$
Thus, $$\Pr_W [ W \subseteq u] > 2^{-r}.$$

Consider the random variable $W' = W \mid (W \subseteq u)$. Since
$u$ is an $(n-1)$-dimensional affine subspace, we can identify $u$ with $\bits^{n-1}$ and think of $W'$ as a random variable over $\A(n-1)$.
Hence, by the inductive hypothesis
(applied with $n-1$ and $r-\frac{1}{2}$),
there exists an affine subspace $s \subseteq u$, such that,
\begin{enumerate}
\item
$$\Pr_{W'} [ W' \subseteq s ] \geq
2^{- \sum_{i=1}^{n-\dim (s) -1 } \left(r - \frac{i}{2} \right)} .$$
\item
$$\left| \Ex_{W' \mid (W' \subseteq s) } [\U_{W'}] - \U_s \right|_1 <
2^{-\left(r-\frac{n}{2}\right)}.$$
\end{enumerate}
We will show that $s$ satisfies the two properties claimed in the statement of the lemma.

For the first property, note that since $s \subseteq u$,
$$\Pr[ W \subseteq s ] = \Pr[ W \subseteq u ] \cdot
\Pr[ W \subseteq s \mid W \subseteq u ]
= \Pr[ W \subseteq u ] \cdot
\Pr[ W' \subseteq s]
$$
$$
> 2^{-r} \cdot 2^{- \sum_{i=1}^{n-\dim (s) - 1} \left(r - \frac{i}{2} \right)}
= 2^{- \sum_{i=0}^{n-\dim (s) -1} \left(r - \frac{i}{2} \right)} .$$

For the second property, note that since $s \subseteq u$,
$$
\Ex_{W \mid (W \subseteq s) } [\U_W] =
\Ex_{W' \mid (W' \subseteq s) } [\U_{W'}].
$$
\end{proof}


The next lemma is the main result of this section.

\begin{lemma} \label{lemma:space-representation}
Let $W \in \A(n)$ be a random variable. Let $r \geq \frac{n}{2}$.
There exists a partial function $\sigma:\A(n) \rightarrow \A(n)$, such that:
\begin{enumerate}
\item \label{lemma:space-representation-prop1}
$\Pr_W [ W \not \in \dom(\sigma) ] \leq 2^{-2n} .$
\item \label{lemma:space-representation-prop2}
For every $w \in \dom(\sigma)$, $w \subseteq \sigma(w)$.
\item \label{lemma:space-representation-prop3}
For every $s \in \im(\sigma)$,
$$\left| \Ex_{W \mid (\sigma(W) = s) } [\U_W] - \U_s \right|_1 <
2^{-\left(r-\frac{n}{2}\right)}.$$
\item \label{lemma:space-representation-prop4}
For every $k \in \N$, there are at most
$$4n \cdot 2^{\sum_{i=0}^{n-k-1} \left(r - \frac{i}{2} \right)}$$ elements
$s \in \im(\sigma)$, with
$\dim(s) \geq k$.
\end{enumerate}
\end{lemma}

\begin{proof}
The proof is by repeatedly applying Lemma~\ref{lemma:subspace}.
We start with the random variable $W_0 = W$, and apply Lemma~\ref{lemma:subspace} on $W_0$. We obtain a subspace $s_0$ (the subspace $s$ whose existence is guaranteed by Lemma~\ref{lemma:subspace}). For every $w \subseteq s_0$, we define $\sigma(w) = s_0$.

We then define the random variable $W_1 = W_0 \mid (W_0 \not \subseteq s_0)$, and apply Lemma~\ref{lemma:subspace} on $W_1$. We obtain a subspace $s_1$ (the subspace $s$ whose existence is guaranteed by Lemma~\ref{lemma:subspace}). For every $w \subseteq s_1$ on which $\sigma$ was still not defined, we define $\sigma(w) = s_1$.

In the same way, in Step~$i$, we define the random variable $W_i = W_{i-1} \mid (W_{i-1} \not \subseteq s_{i-1})$.
Note that $W_i =
W \mid (W \not \subseteq s_{0}) \wedge \ldots \wedge (W \not \subseteq s_{i-1})$, that is, $W_i$ is the restriction of $W$ to the part of $\A(n)$ where $\sigma$ was still not defined.
We apply Lemma~\ref{lemma:subspace} on $W_i$ and obtain a subspace $s_i$ (the subspace $s$ whose existence is guaranteed by Lemma~\ref{lemma:subspace}). For every $w \subseteq s_i$ on which $\sigma$ was still not defined, we define $\sigma(w) = s_i$.

We repeat this until $\Pr_W[ W \not \in \dom(\sigma) ] \leq 2^{-2n}.$

Note that for $i' < i$, $s_{i'} \neq s_i$, because the support of $W_i$ doesn't contain any element $w \subseteq s_{i'}$. Hence, the subspaces $s_0,s_1,\ldots$ are all different.

It remains to show that the four properties in the statement of the lemma hold.

The first property is obvious because we continue to define $\sigma$ on more and more elements repeatedly, until the first property holds.

The second property is obvious because we mapped $w$ to $s_i$ only if $w \subseteq s_i$.

The third property holds by the second property guaranteed by Lemma~\ref{lemma:subspace}.

The forth property holds because by the first property guaranteed by Lemma~\ref{lemma:subspace}, in each step where we obtain a subspace $s_i$
of dimension at least $k$,
we define $\sigma$ on a fraction of at least $2^{- \sum_{i=0}^{n-k-1} \left(r - \frac{i}{2} \right)}$ of the space that still remains. Thus, after at most
$4n \cdot 2^{\sum_{i=0}^{n-k-1} \left(r - \frac{i}{2} \right)}$ such
steps we have $\Pr[ W \not \in \dom(\sigma) ] \leq 2^{-2n}$, and we stop. Thus, the number of elements $s_i$, of dimension at least $k$,
that we obtain in the process, is at most
$4n \cdot 2^{\sum_{i=0}^{n-k-1} \left(r - \frac{i}{2} \right)}$.
\end{proof}

\section{Branching Programs for Parity Learning} \label{section:def}

Recall that
in the problem of parity learning, there is a string $x \in \bits^n$ that was chosen uniformly at random. A learner tries to learn $x$ from  a stream of samples,
$(a_1, b_1), (a_2, b_2) \ldots$, where each~$a_t$ is uniformly distributed over $\bits^n$ and for every $t$, $b_t = a_t \cdot x$.

\subsection{General Branching Programs for Parity Learning}

In the following definition, we model the learner by a {\em branching program}. We allow the branching program to output an affine subspace $w \in \A(n)$. We interpret the output of the program as a guess that $x \in w$.
Obviously, the output $w$ is more meaningful when $\dim(w)$ is relatively small.

\begin{definition} {\bf Branching Program for Parity Learning:}
A branching program of length $m$ and width $d$, for parity learning, is a directed (multi) graph with vertices arranged in $m+1$ layers containing at most $d$ vertices each. In the first layer, that we think of as layer~0, there is only one vertex, called the start vertex.
A vertex of outdegree~0 is called a  leaf.
All vertices in the last layer are leaves
(but there may be additional leaves).
Every non-leaf vertex in the program has $2^{n+1}$ outgoing edges, labeled by elements
$(a,b) \in \bits^n \times \bits$, with exactly one edge labeled by each such $(a,b)$, and all these edges going
into vertices in the next layer.
Each leaf $v$ in the program is labeled by an affine subspace $w(v) \in \A(n)$, that
we think of as the output of the program on that leaf.

{\bf Computation-Path:} The samples
$(a_1, b_1), \ldots, (a_m, b_m) \in \bits^n \times \bits$
that are given as input,
define a
computation-path in the branching
program, by starting from the start vertex
and following at
Step~$t$ the edge labeled by~$(a_t, b_t)$, until reaching a leaf.
The program outputs the label $w(v)$ of the leaf $v$ reached by the computation-path.

{\bf Success Probability:}
The success probability of the program is the probability that $x \in w$, where $w$ is the affine subspace that the program outputs, and the probability is over $x,a_1,\ldots,a_m$ (where $x,a_1,\ldots,a_m$ are uniformly distributed over $\bits^n$, and for every $t$, $b_t = a_t \cdot x$).
\end{definition}

\subsection{Affine Branching Programs for Parity Learning}

Next, we define a special type of a branching program for parity learning, that we call an {\em affine branching program for parity learning}.
In an affine branching program for parity learning, every vertex $v$ (not
necessarily a leaf) is labeled by an affine subspace $w(v) \in \A(n)$.
We will have the property that if the computation-path reaches
$v$ then $x \in w(v)$.
Thus, we can
interpret $w(v)$ as an affine subspace that  is known to contain $x$.

\begin{definition} {\bf Affine Branching Program for Parity Learning:} \label{ABP}
A branching program for parity learning is affine
if each vertex $v$ in the program is labeled by an affine subspace $w(v) \in \A(n)$, and the following properties hold:

\begin{enumerate}
\item
{\bf Start vertex:}
The start vertex is labeled by the space $\bits^n \in \A(n)$.
\item
{\bf Soundness:} For an edge $e=(u,v)$, labeled by $(a,b)$, denote
$$w(e) = w(u) \cap \{x' \in \bits^n : a \cdot x' = b \}.$$
Then,
$$w(e) \subseteq w(v).$$
\end{enumerate}
\end{definition}

Given an affine branching program for parity learning, and samples
$(a_1, b_1), \ldots, (a_m, b_m)$, such that,
for every $t$, $b_t = a_t \cdot x$, it follows by induction that
for every vertex $v$ in the program,
if the computation-path reaches $v$ then
$x \in w(v)$.
In particular, the
output $w$ of the program always satisfies $x \in w$, and thus the success probability of an affine program is always~1.

\subsection{Accurate Affine Branching Programs for Parity Learning}

For a vertex $v$ in a branching program for parity learning,
we denote by $\P_{x|v}$ the distribution of the
random variable~$x$, conditioned on the event that the vertex $v$ was reached by the computation-path.

\begin{definition} {\bf $\epsilon$-Accurate Affine Branching Program for Parity Learning:} \label{AABP}
An affine branching program of length~$m$
for parity learning is $\epsilon$-accurate
if all the leaves are in the last layer, and the following additional property holds
(where $x,a_1,\ldots,a_m$ are uniformly distributed over $\bits^n$, and for every $t$, $b_t = a_t \cdot x$):

\begin{enumerate}
\setcounter{enumi}{2}
\item
{\bf Accuracy:}
Let $0 \leq t \leq m$. Let $V_t$ be the vertex in layer~$t$, reached
by the computation-path.
Let $y_t$ be a random variable uniformly distributed over
the subspace $w(V_t)$,
Then,
$$  \left| \P_{V_t,x}  - \P_{V_t,y_t} \right|_1 \leq \epsilon,$$
or, equivalently,
$$  \Ex_{V_t}  \left| \P_{x|V_t}  - \U_{w(V_t)} \right|_1 \leq \epsilon.$$
\end{enumerate}

\end{definition}

\section{From Branching Programs to Affine Branching Programs} \label{sec:simulation}

In this section, we show that
any branching program $B$ for parity learning can be simulated
by an affine branching program $P$ for parity learning.
Roughly speaking,
each vertex of the simulated program $B$
will be represented by a set of vertices of the simulating program $P$.
Note that the width of $P$
will typically be significantly larger than the width of~$B$.

More precisely,
a branching program $B$ for parity learning is simulated
by a branching program $P$ for parity learning if there exists a mapping
$\Gamma$ from the vertices of $P$ to the vertices of $B$,
and the following properties hold:
\begin{enumerate}
\item
{\bf Preservation of structure:}
For every $i$, $\Gamma$ maps layer $i$ of $P$ to layer $i$ of $B$.
Moreover, $\Gamma$ maps leaves to leaves and non-leaf vertices to non-leaf vertices. Note that $\Gamma$ is not necessarily one-to-one.
\item
{\bf Preservation of functionality:}
For every edge $(u,v)$, labeled by $(a,b)$, in $P$, there is an edge
$(\Gamma(u),\Gamma(v))$, labeled by $(a,b)$, in $B$.
\end{enumerate}

\begin{lemma} \label{lemma:reduction}
Let $k' < n$.
Assume that there exists a length $m$ and width $d$
branching program $B$ for parity learning (of size~$n$), such that:
all leaves of $B$ are in the last layer;
the output of
$B$ is always an affine subspace of dimension $\leq k'$; and the
success probability of $B$ is $\beta$.

Let $\frac{n}{2} \leq r \leq n$. Let
$\epsilon = 4 m \cdot 2^{-\left(r-\frac{n}{2}\right)}$.
Then, there exists an $\epsilon$-accurate length $m$ affine
branching program $P$ for parity learning (of size~$n$), such that:
\begin{enumerate}
\item
For every $k < n$,
the number of vertices in $P$, that are labeled with
an affine subspace of dimension $k$, is at most
$$4n \cdot 2^{\sum_{i=0}^{n-k-1} \left(r - \frac{i}{2} \right)} \cdot dm.$$
\item
For every $k$, such that, $k' < k < n$,
the output of
$P$ is an affine subspace of dimension $< k$, with probability of at least
$$\beta - \epsilon - 2^{-(k-k')}.$$
\end{enumerate}
\end{lemma}

\begin{proof}
For every $0 \leq j \leq m$, let
$\epsilon_j = 4j \cdot 2^{-\left(r-\frac{n}{2}\right)}$.
We will use Lemma~\ref{lemma:space-representation} to turn, inductively,
the layers
of $B$, one by one, into layers of an $\epsilon$-accurate
affine branching program,~$P$.
In Step~$j$ of the induction, we will turn layer~$j$ of $B$ into layer~$j$
of $P$,
and define the label $w(v) \in \A(n)$ for every vertex $v$ in that layer of $P$.
Formally, we will construct, inductively, a sequence of  programs
$B, P_{0}, \ldots , P_{m} = P$, where each program is of length $m$, and for
every~$j$, the program $P_j$ differs from the previous program only in layer $j$
(and in the edges going into layer~$j$ and out of layer~$j$).
After Step~$j$ of the induction, we will have
a branching program $P_{j}$, such that, layers~$0$ to $j$ of $P_{j}$ form an
affine branching program for parity learning.
In addition, the following inductive hypothesis will hold:

\subsubsection*{Inductive Hypothesis:}

Let $\L_j$ be the set of vertices in layer~$j$ of $P_j$.
Let $V_j$ be the vertex in $\L_j$, reached by the computation-path of
$P_j$.
Note that $V_j$ is a random variable that depends on $x,a_1,\ldots,a_j$
(and recall that $x,a_1,\ldots,a_m$ are uniformly distributed over $\bits^n$, and for every $t$, $b_t = a_t \cdot x$).
The inductive hypothesis is that there exists a random variable $U_j$
over $\L_j$, such that, if $y_j$ is a random variable uniformly distributed over
the subspace $w(U_j)$,
then
\begin{equation}\label{Eq:induction}
\left| \P_{V_j,x} - \P_{U_j,y_j} \right|_1 \leq \tfrac {\epsilon_j}{2}.
\end{equation}

The inductive hypothesis is equivalent to the {\em accuracy} requirement (see Definition~\ref{AABP}) for layer~$j$ of $P_j$, up to a
small multiplicative constant in the accuracy, but we need to assume it in this slightly different form, in order to avoid deteriorating the accuracy
by a multiplicative factor
in each step of the induction.

\subsubsection*{Base Case:}

In the base case of the induction, $j=0$, we define $P_{0}$ by
just labeling the start vertex of $B$
by $\bits^n \in \A(n)$.
Thus, the {\em start vertex} property in the definition of an affine branching program is satisfied.
The {\em soundness} property is trivially satisfied because the restriction of  $P_{0}$ to layer~0 contains no edges.
Since we always start from the start vertex,
the distribution of the
random variable~$x$, conditioned on the event that we reached the start vertex, is just~$\U_{n}$, and hence the inductive hypothesis (Equation~\eqref{Eq:induction})
holds with
$U_0 = V_0$.

\subsubsection*{Inductive Step:}

Assume
that we already
turned layers~0 to $j-1$  of $B$ into layers~0 to $j-1$
of~$P$. That is, we already defined the program $P_{j-1}$,
and layers~0 to $j-1$
of $P_{j-1}$ satisfy the  {\em start vertex} property,
the {\em soundness} property, and
the inductive hypothesis (Equation~\eqref{Eq:induction}).
We will now show how to define $P_{j}$ from $P_{j-1}$, that is,
how to turn layer $j$ of $B$ into layer $j$ of $P$.

Let $U_{j-1} \in \L_{j-1}$ be the random variable that satisfies
the inductive hypothesis (Equation~\eqref{Eq:induction}) for layer $j-1$
of $P_{j-1}$.
Let $y_{j-1}$ be a random variable
uniformly distributed over the subspace $w(U_{j-1})$.
Let $a \in_R \bits^n$. Let $b = a \cdot y_{j-1}$.
Let $E=(U_{j-1},V)$ be the edge  labeled by $(a,b)$
outgoing $U_{j-1}$ in~$P_{j-1}$.
Thus, $V$ is a vertex in layer~$j$
of~$P_{j-1}$.
Let $W=w(E)$,
where $w(E)$ is defined as in the {\em soundness} property
in Definition~\ref{ABP}.
That is,
$$w(E) = w(U_{j-1}) \cap \{x' \in \bits^n : a \cdot x' = b \},$$
where $(a,b)$ is the label of $E$, and
$w(U_{j-1})$ is the label of $U_{j-1}$ in $P_{j-1}$.

Let $v$ be a vertex in layer~$j$ of $P_{j-1}$ (and note that $v$ is also
a vertex in layer~$j$ of $B$).
Let
$$W_v = W|(V=v).$$
Let $\sigma_v:\A(n) \rightarrow \A(n)$ be the partial function whose existence is guaranteed by Lemma~\ref{lemma:space-representation}, when applied
on the random variable $W_v$.
Extend $\sigma_v:\A(n) \rightarrow \A(n)$
so that it outputs the special value $*$ on every element
where it was previously undefined.

In the program $P_j$, we will split the vertex $v$ into $|\im (\sigma_v)|$
vertices (where $\im (\sigma_v)$ already contains the additional special value $*$).
For every $s \in \im (\sigma_v)$, we will have a vertex $(v,s)$.
If $s \neq *$, we label the vertex $(v,s)$ by the affine subspace $s$,
and we label the additional vertex $(v,*)$
by $\bits^n$.
For every $s \in \im (\sigma_v)$, the edges going out of $(v,s)$ (in~$P_j$) will be the same
as the edges going out of $v$ in $P_{j-1}$. That is, for every edge $(v,v')$
(from layer~$j$ to layer~$j+1$)
in the program $P_{j-1}$, and every $s \in \im (\sigma_v)$, we will have an edge $((v,s),v')$ with the same label,
(from layer~$j$ to layer~$j+1$) in the program $P_{j}$.

We will now define the edges going into the vertices $(v,s)$
in the program~$P_j$.
For every edge $e=(u,v)$, labeled by $(a,b)$, (from layer~$j-1$ to layer~$j$), in the program $P_{j-1}$,
consider the affine subspace
$w = w(e) = w(u) \cap \{x' \in \bits^n : a \cdot x' = b \}$
(as in the {\em soundness} property in Definition~\ref{ABP}),
where $w(u)$ is the label of $u$ in $P_{j-1}$.
Let $s = \sigma_v(w)$.

In $P_{j}$ , we will have
the edge $(u,(v,s))$
(labeled by $(a,b)$), from layer~$j-1$ to layer~$j$, that is, we connect $u$ to $(v,s)$.
Note that the edge $(u,(v,s))$ satisfies
the {\em soundness} property in the definition of an affine branching program:
If $s \neq *$,
the vertex $(v,s)$ is labeled by $s = \sigma_v(w)$ and by
Poperty~\ref{lemma:space-representation-prop2} of Lemma~\ref{lemma:space-representation}, $w \subseteq \sigma_v(w)$.
If $s = *$, the vertex $(v,s)$ is labeled by $\bits^n$
and hence the {\em soundness} property is
trivially satisfied.

\subsubsection*{Proof of the Inductive Hypothesis:}

Next, we will  prove the
inductive hypothesis (Equation~\eqref{Eq:induction}), for $P_{j}$.
We will define the random variable $U_{j} \in \L_{j}$ as follows:

As before, let $U_{j-1} \in \L_{j-1}$ be the random variable that satisfies
the inductive hypothesis (Equation~\eqref{Eq:induction}) for layer $j-1$
of $P_{j-1}$.
Let $y_{j-1}$ be a random variable uniformly distributed over the subspace $w(U_{j-1})$.
Let $a \in_R \bits^n$. Let $b = a \cdot y_{j-1}$.
Let $E=(U_{j-1},V)$ be the edge  labeled by $(a,b)$
outgoing $U_{j-1}$ in~$P_{j-1}$.
Thus, $V$ is a vertex in layer~$j$
of~$P_{j-1}$.
As before, let $W=w(E) = w(U_{j-1}) \cap \{x' \in \bits^n : a \cdot x' = b \}$.
As before, for a vertex
$v$ in layer~$j$ of $P_{j-1}$,
let $\sigma_v:\A(n) \rightarrow \A(n)$ be the partial function whose existence is guaranteed by Lemma~\ref{lemma:space-representation}, when applied
on the random variable $W_v = W|(V=v)$, and
extend $\sigma_v:\A(n) \rightarrow \A(n)$
so that it outputs the special value $*$ on every element
where it was previously undefined.

We define $U_{j} = (V,\sigma_V(W)) \in \L_{j}$.
Let
$y_j$ be a random variable uniformly distributed over the subspace $w(U_j)$,
and let $V_j$ be the vertex in $\L_j$, reached by the computation-path of
$P_j$.
We need to prove that
\begin{equation}\label{Eq0}
\left| \P_{V_j,x} - \P_{U_j,y_j} \right|_1 \leq
2j \cdot 2^{-\left(r-\frac{n}{2}\right)}.
\end{equation}

Let $y'_j$ be a random variable uniformly distributed over the subspace $W$.
Equation~\eqref{Eq0} follows by the following two equations and by the triangle
inequality:
\begin{equation}\label{Eq1}
\left| \P_{U_j,y'_j} - \P_{U_j,y_j} \right|_1 \leq
2 \cdot 2^{-\left(r-\frac{n}{2}\right)}.
\end{equation}
\begin{equation}\label{Eq2}
\left| \P_{V_j,x} - \P_{U_j,y'_j} \right|_1 \leq
2(j-1) \cdot 2^{-\left(r-\frac{n}{2}\right)}.
\end{equation}
Thus, it is sufficient to prove Equation~\eqref{Eq1} and Equation~\eqref{Eq2}.
We will start with Equation~\eqref{Eq1}.

By Property~\ref{lemma:space-representation-prop3} of Lemma~\ref{lemma:space-representation},
for every $v$ in layer~$j$ of $P_{j-1}$, and every
$s \in \im(\sigma_v) \setminus \{*\}$,
$$\left| \Ex_{W \mid (V=v) , (\sigma_v(W) = s) } [\U_W] - \U_s \right|_1 <
2^{-\left(r-\frac{n}{2}\right)}.$$
By the definitions of  $y'_j$ and $U_j$,
$$
\Ex_{W \mid (V=v) , (\sigma_v(W) = s) } [\U_W] =
\Ex_{W \mid (U_j = (v,s)) } [\U_W] =
\P_{y'_j | (U_j = (v,s))}.
$$
By the definition of $y_j$,
$$
\U_s =
\P_{y_j | (U_j = (v,s))}
$$
Hence
$$
\left| \P_{y'_j | (U_j = (v,s))} - \P_{y_j | (U_j = (v,s))} \right|_1 <
2^{-\left(r-\frac{n}{2}\right)}.
$$

Taking expectation over $U_j$, and taking into account that,
by Property~\ref{lemma:space-representation-prop1} of Lemma~\ref{lemma:space-representation}, for every $v$,
 $\Pr (\sigma_v(W) = *) \leq 2^{-2n}$,
we obtain
\begin{equation*}
\left| \P_{U_j,y'_j} - \P_{U_j,y_j} \right|_1 =
\Ex_{U_j}
\left| \P_{y'_j | U_j} - \P_{y_j | U_j} \right|_1 <
2^{-\left(r-\frac{n}{2}\right)} + 2^{-2n},
\end{equation*}
which proves Equation~\eqref{Eq1}.

We will now prove Equation~\eqref{Eq2}.
Let $\T$ be the following probabilistic transformation from
$\L_{j-1} \times \bits^n$ to $\L_{j} \times \bits^n$.
Given $(u,z) \in \L_{j-1} \times \bits^n$, the transformation $\T$ chooses
$a \in_R \bits^n$ and  $b = a \cdot z$, and outputs $(V,z)$, where
$V \in \L_{j}$ is the vertex obtained by following the edge labeled
by $(a,b)$
outgoing $u$ in~$P_{j}$.

By the definition of the computation-path, $\T (V_{j-1},x)$ has
the same distribution as $(V_j,x)$.
By the definition of $U_j,y_j,y'_j$, we have that $\T (U_{j-1},y_{j-1})$ has
the same distribution as $(U_j,y'_j)$.
Hence, by the triangle inequality and the inductive hypothesis,
\begin{equation*}
\left| \P_{V_j,x} - \P_{U_j,y'_j} \right|_1 =
\left| \P_{\T (V_{j-1},x)} - \P_{\T (U_{j-1},y_{j-1})} \right|_1 \leq
\left| \P_{V_{j-1},x} - \P_{U_{j-1},y_{j-1}} \right|_1 \leq
2(j-1) \cdot 2^{-\left(r-\frac{n}{2}\right)},
\end{equation*}
which gives Equation~\eqref{Eq2}.

Since, by induction, layers~$0$ to $j-1$ of $P_{j-1}$ form an
affine branching program for parity learning, and since we already saw that
all the edges between layer~$j-1$ and layer~$j$ of $P_{j}$ satisfy
the {\em soundness} property in the definition of an affine branching program,
we have that
layers~$0$ to $j$ of $P_{j}$ form an
affine branching program for parity learning.

\subsubsection*{$P$ is $\epsilon$-Accurate:}

We will now prove that the final branching program $P=P_m$, that we obtained, satisfies
the requirements of the lemma.
We already know that $P$ is an affine branching program for parity learning.

We will start by proving that $P$ is $\epsilon$-accurate.
Let $0 \leq t \leq m$. Let $V_t$ be the vertex in layer~$t$ of $P$, reached
by the computation-path of $P$.
Let $z_t$ be a random variable uniformly distributed over
the subspace $w(V_t)$,
We need to prove that,
\begin{equation} \label{eq:ac}
\left| \P_{V_t,x}  - \P_{V_t,z_t} \right|_1 \leq \epsilon.
\end{equation}
Recall that by the inductive hypothesis (Equation~\eqref{Eq:induction}),
there exists a random variable $U_t$
over layer~$t$ of $P$, such that, if $y_t$ is a random variable uniformly distributed over
the subspace $w(U_t)$,
then
\begin{equation} \label{eq:ac1}
\left| \P_{V_t,x} - \P_{U_t,y_t} \right|_1 \leq \tfrac {\epsilon}{2},
\end{equation}
and this also implies
\begin{equation*}
\left| \P_{V_t} - \P_{U_t} \right|_1 \leq \tfrac {\epsilon}{2}.
\end{equation*}
By the last inequality and since for every $v$ in layer~$t$ of $P$, it holds that
$\P_{z_t | (V_t=v)} = \P_{y_t | (U_t=v)} $
(since they are both uniformly distributed over $w(v)$), we have
\begin{equation} \label{eq:ac2}
\left| \P_{V_t,z_t} - \P_{U_t,y_t} \right|_1 =
\left| \P_{V_t} - \P_{U_t} \right|_1
\leq \tfrac {\epsilon}{2}.
\end{equation}
Equation~\eqref{eq:ac} follows by Equation~\eqref{eq:ac1}, Equation~\eqref{eq:ac2} and
the triangle inequality.

\subsubsection*{$P$ Satisfies the Additional Properties:}

We will now prove that $P$ satisfies the two additional properties claimed
in the statement of the lemma.
The first property holds since
Property~\ref{lemma:space-representation-prop4} of
Lemma~\ref{lemma:space-representation} ensures that for every vertex in layers 1 to $m$ of the branching program $B$, we obtain at most
$4n \cdot 2^{\sum_{i=0}^{n-k-1} \left(r - \frac{i}{2} \right)}$ vertices
in the branching program $P$ that are labeled with affine subspaces of dimension~$k$.

It remains to prove the second property.
Let $V_m = (V,S)$ be the vertex in layer $m$ of $P$, reached by the computation-path of
$P$.
Note that $V_m$ is a random variable that depends on $x,a_1,\ldots,a_m$
(and recall that $x,a_1,\ldots,a_m$ are uniformly distributed over $\bits^n$, and for every $t$, $b_t = a_t \cdot x$).

Note that $V$ is the vertex in layer $m$ of $B$, reached by the
computation-path of $B$ (on the same $x,a_1,\ldots,a_m$).
This is true since $P$ simulates $B$. More precisely, by the construction,
if on $x,a_1,\ldots,a_m$, the program $P$ reaches $(V,S)$, then, on the same
$x,a_1,\ldots,a_m$, the program $B$ reaches $V$.

Since the success probability of $B$ is $\beta$,
$$\Pr [ x \in w(V)] = \beta,$$
where $w(V)$ is the label of $V$ in $B$.
Let $y_m$ be a random variable uniformly distributed over
the subspace $w(V_m)$,
where $w(V_m)$ is the label of $V_m$ in $P$.
Since $P$ is $\epsilon$-accurate,
$$  \left| \P_{V,x}  - \P_{V,y_m} \right|_1 \leq
\left| \P_{V,S,x}  - \P_{V,S,y_m} \right|_1 =
\left| \P_{V_m,x}  - \P_{V_m,y_m} \right|_1 \leq \epsilon.$$
Thus,
$$\Pr [ y_m \in w(V)] \geq \Pr [ x \in w(V)] - \epsilon =
\beta - \epsilon.$$

Let $k > k'$.
Recall that $w(V)$ is of dimension $ \leq k'$. Thus,
if $w(V_m)$ is of dimension $\geq k$, the (conditional) probability that
$y_m \in w(V)$ is at most $2^{k'-k}$. Thus,
$$\beta - \epsilon \leq \Pr [ y_m \in w(V)] \leq
\Pr [ \dim(w(V_m)) < k]+ 2^{k'-k}.$$
That is,
$$
\Pr [ \dim(w(V_m)) < k] \geq \beta - \epsilon - 2^{-(k-k')}.$$
\end{proof}

\section{Time-Space Lower Bounds for Parity Learning} \label{sec:TMproof}

In this section, we will use Lemma~\ref{lemma:reduction} to
prove Theorem~\ref{thm:TM2}, our main result.
Recall that Theorem~\ref{thm:TM2} is stronger than
Theorem~\ref{thm:TM1}, and hence Theorem~\ref{thm:TM1} follows as well.
We start by a lemma that will be used, in the proof of Theorem~\ref{thm:TM2},
to obtain time-space lower bounds for
{\em affine} branching programs.

\begin{lemma}  \label{lemma:probaffine}
Let $k < n$.
Let $P$ be a length $m$ affine
branching program for parity learning (of size~$n$), such that, for every
vertex $u$ of $P$, $\dim(w(u)) \geq k$.
Let $v$ be a vertex of $P$, such that, $\dim(w(v)) = k$.
Then, the probability that the computation-path of $P$ reaches $v$ is at most
$$m^{n-k} \cdot 2^{\sum_{j=0}^{n-k-1} \left( n- 2k - j \right) }.$$
\end{lemma}

\begin{proof}
Let $s$ be the vector space ``orthogonal'' to
$w(v)$ in $\bits^n$. That is,
$$
s = \left\{
a \in \bits^n : \exists b \in \bits \; \forall x' \in w(v): a \cdot x'  = b
\right\}.
$$

Let $V_0, \ldots , V_m$ be the vertices on the computation-path of $P$.
Note that $V_0, \ldots , V_m$ are random variables that depend on
$x,a_1,\ldots,a_m$.
For every $0 \leq i \leq m$, let $S_i$ be the vector space ``orthogonal'' to
$w(V_i)$ in $\bits^n$. That is,
$$
S_i = \left\{
a \in \bits^n : \exists b \in \bits \; \forall x' \in w(V_i): a \cdot x'  = b
\right\}.
$$
By the {\em soundness} property
in Definition~\ref{ABP}, for every $1 \leq i \leq m$,
\begin{equation} \label{eq:span}
S_i \subseteq \mbox{span} (S_{i-1} \cup a_i).
\end{equation}

For every $0 \leq i \leq m$,
let $Z_i = \dim(S_i \cap s)$.
Note that $Z_0 =0$, and by Equation~\eqref{eq:span}, for every $1 \leq i \leq m$, $Z_i \leq Z_{i-1}+1$.
If the computation-path of $P$ reaches $v$ then for some
$1 \leq i \leq m$, $Z_{i}= n-k$.
Thus,
if the computation-path of $P$ reaches $v$,
there exist $n-k$ indices $i_1 < \ldots < i_{n-k} \in [m]$, such that,
the following event, denoted by $E_{i_1,\ldots,i_{n-k}}$, occurs:
$$
E_{i_1,\ldots,i_{n-k}} = \bigwedge_{j \in [n-k]} (Z_{i_j -1} =j-1) \wedge (Z_{i_j} =j).
$$
(In particular, $E_{i_1,\ldots,i_{n-k}}$ occurs if for every $j$, we take
$i_j$ to be the first
$i$ such that $Z_{i} =j$).
We will bound the probability that
the computation-path of $P$ reaches $v$, by bounding
$\Pr [ E_{i_1,\ldots,i_{n-k}}]$, and taking the union bound over (less than)
$m^{n-k}$
possibilities for $i_1,\ldots,i_{n-k} \in [m]$.

Fix
$i_1 < \ldots < i_{n-k} \in [m]$.
For $r \in \{0,\ldots,{n-k}\}$, let
$$
E_{i_1,\ldots,i_r} = \bigwedge_{j \in [r]} (Z_{i_j -1} =j-1) \wedge (Z_{i_j} =j).
$$
Thus,
$$\Pr [ E_{i_1,\ldots,i_{n-k}}]=
\prod_{j \in [{n-k}]} \Pr [ E_{i_1,\ldots,i_j} \mid E_{i_1,\ldots,i_{j-1}}].$$
We will show how to bound
$\Pr [ E_{i_1,\ldots,i_j} \mid E_{i_1,\ldots,i_{j-1}}]$.
$$
\Pr [ E_{i_1,\ldots,i_j} \mid E_{i_1,\ldots,i_{j-1}}] =
\Pr [ (Z_{i_j -1} =j-1) \wedge (Z_{i_j} =j) \mid E_{i_1,\ldots,i_{j-1}}]
$$
$$
=
\Pr [ (Z_{i_j -1} =j-1) \wedge (Z_{i_j -1} < Z_{i_j}) \mid E_{i_1,\ldots,i_{j-1}}]
$$
\begin{equation} \label{eq:EEEE}
\leq
\Pr [ (Z_{i_j -1} < Z_{i_j}) \mid
E_{i_1,\ldots,i_{j-1}} \wedge (Z_{i_j -1} =j-1)].
\end{equation}
Note that the event $E_{i_1,\ldots,i_{j-1}} \wedge (Z_{i_j -1} =j-1)$
that we condition on, on the right hand side,
depends only on
$x,a_1, \ldots, a_{i_j-1}$.
We will bound the probability for the event $(Z_{i_j -1} < Z_{i_j})$,
conditioned on any event that fixes $Z_{i_j -1}$ and
depends only on
$x,a_1, \ldots, a_{i_j-1}$.

More generally,
fix $1 \leq i \leq m$, and let $E'_i$ be the event $(Z_{i-1} < Z_i)$.
Let $E'$ be any event that fixes $Z_{i-1}$ and depends only on
$x,a_1, \ldots, a_{i-1}$. Without loss of generality, we can assume that
the event $E'$ just fixes the values of
$x,a_1, \ldots, a_{i-1}$.
We will show how to bound $\Pr [ E'_i \mid E']$.

Thus, we
fix
$x,a_1, \ldots, a_{i-1}$ and
we will bound $\Pr [ E'_i]$ (conditioned on $x,a_1, \ldots, a_{i-1}$).
By Equation~\eqref{eq:span},
if $E'_i$ occurs then
$\dim(S_{i-1} \cap s) < \dim(S_i \cap s) \leq
\dim(\mbox{span} (S_{i-1} \cup a_i) \cap s)$,
and hence $S_{i-1} \cap s  \subsetneq \mbox{span} (S_{i-1} \cup a_i) \cap s$,
which implies that there exists
$a \in S_{i-1}$, such that, $a \oplus a_i \in s$.
For every fixed $a \in S_{i-1}$, the event $a \oplus a_i \in s$ occurs
with probability $2^{\dim(s) -n} = 2^{(n-k)-n} = 2^{-k}$
(since $a_i$ is uniformly distributed and independent
of $x,a_1, \ldots, a_{i-1}$).
We will bound the probability for $E'_i$ by
taking  a union bound over all possibilities for $a$, but doing so we
take into account that $a \in S_{i-1}$ satisfies $a \oplus a_i \in s$
if and only if
every $a'  \in a \oplus (S_{i-1} \cap s)$ satisfies $a' \oplus a_i \in s$.
Thus, we can take a union bound over
$2^{\dim(S_{i-1}) - Z_{i-1}} \leq 2^{n- k - Z_{i-1}}$ possibilities
(where we assume that $Z_{i-1}$ is fixed).
Hence, by the union bound
\begin{equation*}
\Pr [E'_i \mid E' ] \leq 2^{n- k - Z_{i-1}} \cdot 2^{-k}= 2^{n- 2k - Z_{i-1}}.
\end{equation*}

Thus, in particular, by Equation~\eqref{eq:EEEE},
\begin{equation*}
\Pr [ E_{i_1,\ldots,i_j} \mid E_{i_1,\ldots,i_{j-1}}]
\leq  2^{n- 2k - (j-1)}.
\end{equation*}
Hence,
\begin{equation*}
\Pr [ E_{i_1,\ldots,i_{n-k}}]
\leq \prod_{j \in [{n-k}]}
2^{n- 2k - (j-1)}=
2^{\sum_{j=0}^{n-k-1} \left( n- 2k - j \right) }.
\end{equation*}

By the union bound,
the probability that the computation-path of $P$ reaches $v$ is at most
\begin{equation*}
m^{n-k} \cdot 2^{\sum_{j=0}^{n-k-1} \left( n- 2k - j \right) }.
\end{equation*}
\end{proof}

\begin{theorem} \label{thm:TM2}
For any $c < \frac{1}{20}$, there exists $\alpha >0$, such that the following holds:
Let $B$ be a branching program of
length at most $2^{\alpha n}$ and width at most $2^{c n^2}$
for parity learning (of size~$n$), such that,
the
output of $B$ is always an affine subspace of dimension $\leq \frac{3}{5} n$.
Assume for simplicity and without loss of generality that all leaves
of $B$ are in the last layer.
Then,
the success probability of $B$
(that is, the probability that $x$ is contained in the subspace that $B$ outputs)
is at most $O(2^{-\alpha n})$.
\end{theorem}

\begin{proof}
Let $0 < \alpha < \frac{1}{5}$ be a sufficiently small constant (to be determined later on).
Let $B$ be a branching program of
length $m=2^{\alpha n}$ and width $d=2^{c n^2}$
for parity learning (of size~$n$), such that,
the
output of $B$ is always an affine subspace of dimension $\leq \frac{3}{5} n$.
Assume for simplicity and without loss of generality that all leaves
of $B$ are in the last layer.
Denote by $\beta$ the success probability of $B$.

Let $r = \left(\frac{1}{2}+2\alpha\right) \cdot n$.
Let $k = \frac{4}{5} n$.
By Lemma~\ref{lemma:reduction}, there exists a length $m$ affine
branching program $P$ for parity learning (of size~$n$), such that:
\begin{enumerate}
\item
The number of vertices in $P$, that are labeled with
an affine subspace of dimension $k$, is at most
$$4n \cdot 2^{\sum_{i=0}^{n-k-1} \left(r - \frac{i}{2} \right)} \cdot dm.$$
\item
The output of
$P$ is an affine subspace of dimension $\leq k$, with probability of at least
$$\beta - 4 \cdot 2^{-\alpha n} - 2^{-\frac{1}{5} n}
\geq  \beta - 5 \cdot 2^{-\alpha n}.$$
\end{enumerate}

Assume without loss of generality that
every
vertex $u$ of $P$, such that  $\dim(w(u)) = k$, is a leaf.
(Otherwise, we can just redefine $u$ to be a leaf by removing all the edges
going out of it).
Assume without loss of generality that
for every
vertex $u$ of $P$, $\dim(w(u)) \geq k$.
(Otherwise, we can just remove $u$ as it is unreachable from the start vertex,
since we defined all vertices labeled by subspaces of dimension $k$ to be leaves
and since by the {\em soundness} property
in Definition~\ref{ABP}, the dimensions along the computation-path can only
decrease by 1 in each step).

By Lemma~\ref{lemma:probaffine}, and by substituting the values of $m,d,k,r$,
the probability that the computation-path of $P$ reaches some vertex
that is labeled with an affine subspace of dimension $k$ is at most
$$
\left(4n \cdot 2^{\sum_{i=0}^{n-k-1} \left(r - \frac{i}{2} \right)} \cdot dm
\right) \cdot
\left(
m^{n-k} \cdot 2^{\sum_{i=0}^{n-k-1} \left( n- 2k - i \right) } \right)$$
$$
=
4nm \cdot 2^{c n^2} \cdot
\left( 2^{\sum_{i=0}^{n-k-1}
\left( \frac{1}{2}n+2\alpha  n - \frac{i}{2} \right)}
\right) \cdot
\left(
2^{\alpha n (n-k)} \cdot 2^{\sum_{i=0}^{n-k-1}
\left( -\frac{3}{5} n  - i \right) } \right)$$
$$
=
4nm \cdot 2^{c n^2} \cdot
2^{(n-k)\left(3\alpha n -\frac{1}{10} n\right) } \cdot
\left( 2^{\sum_{i=0}^{n-k-1}
\left( - \frac{3}{2}i \right)}
\right)
$$
$$
=
4nm \cdot 2^{c n^2} \cdot
2^{(n-k)\left(3\alpha n -\frac{1}{10} n\right) } \cdot
2^{- \frac{3}{4} (n-k)\cdot (n-k-1)}
$$
$$
=
4nm \cdot 2^{c n^2} \cdot
2^{\frac{1}{5}n\left(3\alpha n -\frac{1}{10} n -\frac{3}{20} n +
\frac{3}{4}\right) }
$$
$$
=
4nm \cdot 2^{n^2 \left(c + \frac{3}{5}\alpha -\frac{1}{20}
+ \frac{3}{20n}\right)}.
$$
Thus, if $\alpha < \frac{5}{3} \left(\frac{1}{20} - c \right)$,
this probability is at most $2^{-\Omega(n^2)}$, and
hence,
$$
\beta - 5 \cdot 2^{-\alpha n} \leq 2^{-\Omega(n^2)}.
$$
That is,
$$
\beta \leq  O( 2^{-\alpha n}).
$$
\end{proof}

\end{document}